\documentclass[10pt]{article} 
\usepackage[preprint]{tmlr}

\usepackage{amsmath,amsfonts,bm}

\def\eqref#1{equation~\ref{#1}}

\def\1{\bm{1}}

\DeclareMathAlphabet{\mathsfit}{\encodingdefault}{\sfdefault}{m}{sl}
\SetMathAlphabet{\mathsfit}{bold}{\encodingdefault}{\sfdefault}{bx}{n}

\usepackage{hyperref}
\usepackage{url}

\usepackage{verbatim}
\usepackage{graphicx}
\usepackage{amssymb}
\usepackage{amsthm}
\usepackage{enumitem}
\usepackage{multirow}

\title{Elementwise Language Representation}

\author{\name Dunam Kim \email timecostslives@gmail.com \\
      \addr Independent researcher
      \AND
      \name Jeeeun Kim \email topdogje@naver.com \\
      \addr Pohang University of Science and Technology}

\def\month{MM}  
\def\year{YYYY} 
\def\openreview{\url{https://openreview.net/forum?id=XXXX}} 

\newtheorem*{theorem*}{Assumption}
\newtheorem*{prop}{Proposition}

\begin{document}

\maketitle

\begin{abstract}
We propose a new technique for computational language representation called elementwise embedding,
in which a material (semantic unit) is abstracted into a horizontal concatenation of lower-dimensional element
(character) embeddings. While elements are always characters, materials are arbitrary levels of semantic
units so it generalizes to any type of tokenization. To focus only on the important letters, the $n^{th}$
spellings of each semantic unit are aligned in $n^{th}$ attention heads, then concatenated back
into original forms creating unique embedding representations; they are jointly projected thereby determining own contextual importance.
Technically, this framework is achieved by passing a sequence of materials, each consists of $v$ elements,
to a transformer having $h=v$ attention heads. As a pure embedding technique, elementwise embedding
replaces the $w$-dimensional embedding table of a transformer model with $256$ $c$-dimensional elements 
(each corresponding to one of UTF-8 bytes) where $c=w/v$. Using this novel approach, we show that the 
standard transformer architecture can be reused for all levels of language representations and be able to 
process much longer sequences at the same time-complexity without "any" architectural modification and additional overhead.
BERT trained with elementwise embedding outperforms its subword equivalence 
(original implementation) in multilabel patent document classification exhibiting superior robustness to domain-specificity
and data imbalance, despite using $0.005\%$ of embedding parameters.
Experiments demonstrate the generalizability of the proposed method by successfully transferring these enhancements to
differently architected transformers CANINE and ALBERT.
\end{abstract} 

\section{Introduction}
We understand texts from various levels of semantics but current language representation strategies leverage
tokenization which relies on a certain level of semantics exclusively, fully ignoring the hierarchical structures of natural languages. 
Text is encoded to a sequence of integers then projected into fixed-size latent embeddings.
These types of expressions result in a recursive trade-off between different levels of language representations: (sub)word-level models
indirectly recover characters \citep{itzhak2021models} but it is not always sufficient for spelling-sensitive tasks,
character-level models need much longer sequences to reach comparable performance to word-level models thus amplifying the
computational complexity of self-attention. Some recently proposed studies \citep{clark2022canine,godey2022manta,tay2021charformer}
attempt to solve this by downsampling long character sequences into an acceptable length, however, they share the same limitation as pure
character-level models because their valid downsampling rates are constrained to relatively small values mainly due to the smoothing and overhead issues.

Instead, we propose elementwise embedding, a language representation technique for addressing this trade-off in which a set of lower-dimensional 
character embeddings called \emph{elements} are horizontally concatenated into a single latent embedding called \emph{material} that mimics a semantic unit such as 
a word, phrase, sentence and etc. Using this method, models with higher-dimensional hidden representations create each semantic unit (i.e., a material) by concatenating a greater numbers of characters (i.e., elements), which implies that larger models can process longer sequences than smaller ones at the same computational complexity. This means that the acceptable sequence length scales with the size of a transformer model, but the complexity is fixed as that of its attention. Assuming that a character-level GPT-3 [processing 2048 12,288-dimensional token embeddings with 96 attention heads; \citet{brown2020language}] is trained with elementwise embedding, it aligns a sequence of $2,048\times96=296,608$ characters which is 96x longer at the same $O(N\sqrt{N})_{N=2048}$ complexity. 

The proposed methodology follows the two-step framework of \emph{"reshape, then focus"}. First, the given text is encoded as a sequence of $uv$ UTF-8 bytes and projected into a $(uv, c)$ embedding matrix in which each row is a $c$-dimensional element; it's \emph{"reshaped"} into a $(u, w)$ embedding matrix in which each row is a $w$-dimensional material (e.g., a word), where $c=w/v$. As a result, one material consists of $v$ elements so that we can align $uv$ elements at the $O(u^{2})$ complexity using multihead self-attention \citep{vaswani2017attention} with $v$ attention heads. Each $i^{th}$ column of this $(u, v)$ material matrix is a sequence of the $i^{th}$ elements of all $u$ materials, so $i^{th}$ attention head aligns $i^{th}$ elements. This operation is most straightforward when a material is a $v$ letters word: $i^{th}$ spellings of all $u$ words are aligned in $i^{th}$ attention head, then concatenated back creating unique embedding representations where $i\in[1,v]$. Each attended $i^{th}$ spelling is referred as \emph{focus} because it is quite similar to that we often read text inferring the meanings of words by \emph{"focusing"} on a few important letters. The contextual importance of each word is determined jointly via linear transformation. Theoretically, this can be understood as lowering the entropy of character sequences concentrating distributed probabilities into several important spellings. Technically, it is just to pass a $(u, w)$ word embedding matrix in which each row is a horizontal concatenation of $v$ $c$-dimensional character embeddings as input to a transformer model with $w$-dimensional hidden layers.
It's identical to aligning words using character-level semantics and vice versa.

In practical implementation, focus is performed by multihead attention of the parent (any transformer model) by setting the number of attention heads to $h=v$, so applying elementwise embedding is simply to replace the embedding table of parent model with a set of 256 $c$-dimensional character embeddings (each mapping to one of UTF-8 bytes; elements) and a following tensor reshaping operation. Neither structural modification of neural networks nor additional operations such as up/downsampling that entail unnecessary engineering efforts and overheads are required. Fig \ref{fig1} offers an intuitive visualization of elementwise embedding.

\begin{figure}[t!]
\begin{center}
\includegraphics[width=1.0\linewidth]{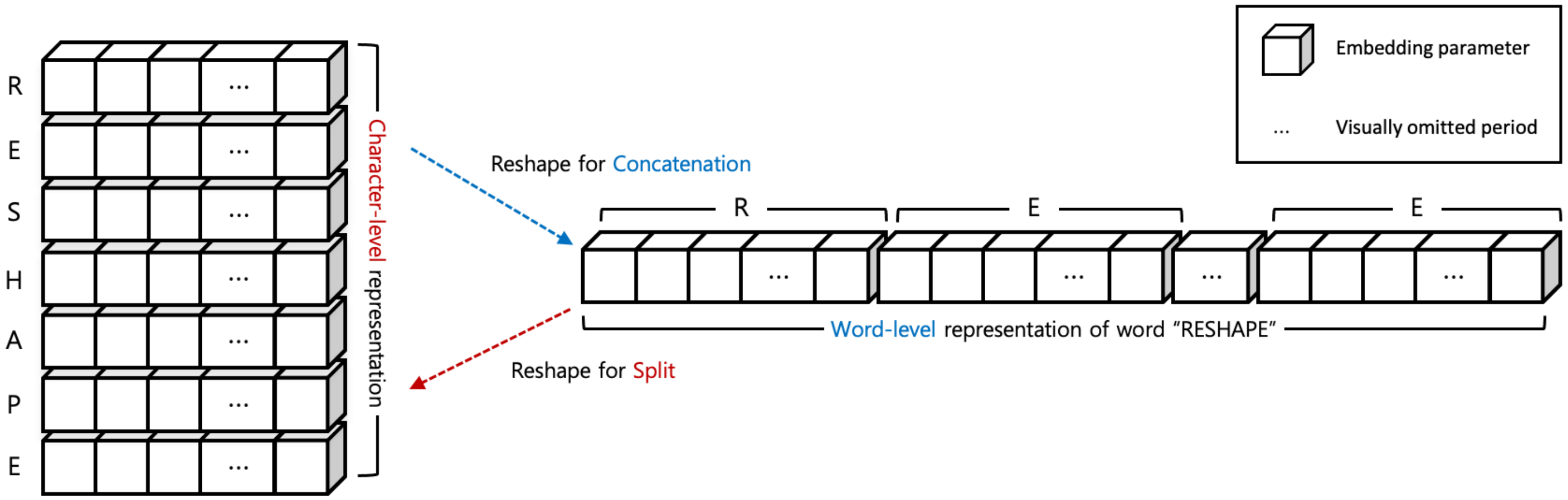}
\end{center}
\caption{Visualization of the proposed elementwise language representation. A material (word "RESHAPE" here)
is abstracted into a horizontal concatenation of elements (spellings [R, E, S, H, A, P, E] here).}
\label{fig1}
\end{figure}

\section{Research Objectives}
In this study, we suggest the new elementwise language representation and demonstrate its validity.

Theoretically, we propose the first generalized language representation:
\begin{itemize}[topsep=0pt,noitemsep]
\item applying with all levels of tokenization strategies
\item aligning longer sequences proportional to the model's size
\item based on information theory rather than the distributional hypothesis
\end{itemize}

Empirically, we demonstrate the practical contributions of the proposed methodology by:

\begin{itemize}[topsep=0pt,noitemsep]
\item reusing BERT \citep{devlin2018bert} for various levels of language representation
\item improving BERT to be more robust to domain-specific and imbalanced training examples
\item improving BERT to process longer sequences at the same $O(N^2)$ computational complexity
\end{itemize}
without any architectural modification and additional overhead. 

We generalize these contributions to different transformers CANINE \citep{clark2022canine} and ALBERT \citep{lan2019albert}. Through experiments, we validate the clear superiority in robustness to domain-specificity and dataset imbalance of the proposed method by comparing transformers, trained with elementwise embedding from scratch without pretraining, with their original implementations for multilabel patent classification.

This is the first part of our two-paper study discussing:
\begin{enumerate}[topsep=0pt,noitemsep]
\item Theoretical and practical advantages of elementwise language representation
\item Unsupervised language modeling strategy for elementwise language representation
\end{enumerate}
Though these were originally topics to be covered in one paper, we divided them into two parts mainly due to the lack of computational resources for pretraining and demonstrating language models in various
scales at the time of this study.

\begin{figure}[h]
\begin{center}
\includegraphics[width=1.0\linewidth]{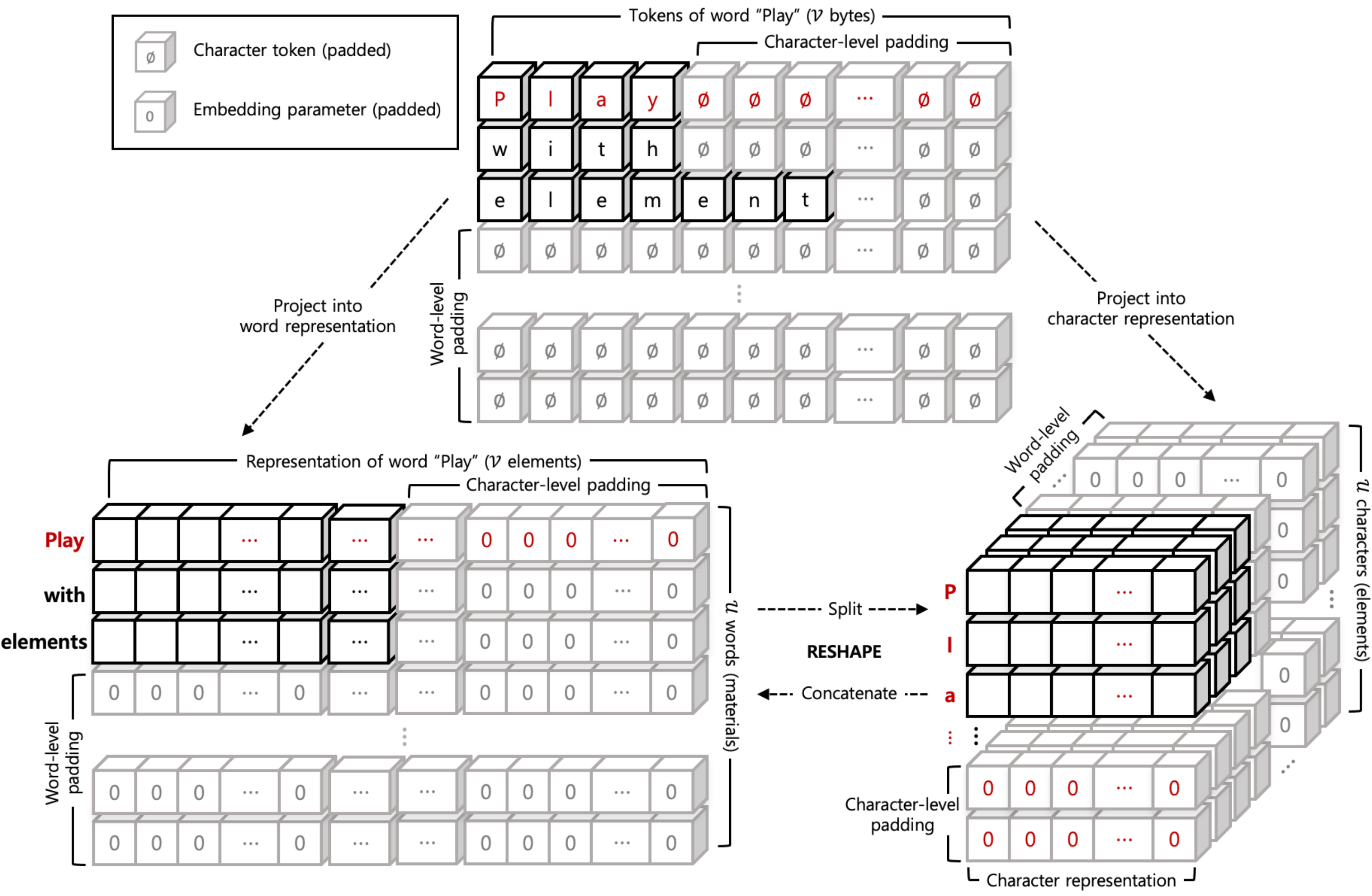}
\end{center}
\caption{Overall framework of elementwise language representation applied with whitespace tokenization. First, the given text is tokenized into a sequence of $u$ words based on whitespaces.
Each word is encoded to a sequence of $v$ UTF-8 bytes resulting in a $uv$ bytes sequence; each byte is 4 greater than its original value for reserving 4 special tokens [CLS], [SEP], [PAD] and [MASK]; the token [MASK] is reserved for unsupervised pretraining with elementwise embedding to be introduced in our follow-up study. Words shorter than $v$ are padded with integer zeros, longer ones are truncated. Sequences shorter than $u$ are padded with embeddings filled with $v$ zeros. $uv$ bytes are projected into a $(uv, c)$ character embedding matrix, and then reshaped into a $(u, w)$ word embedding matrix in which each row comprises of horizontally concatenated $v$ $c$-dimensional character embeddings, where $c=w/v$. Transformation from (sub)word-level to character-level representation and vice versa are always possible via reshaping operation. This framework can be generalized to all kinds of tokenization; just split text into $u$ tokens, encode each token to $v$ bytes, project, then reshape.}
\label{fig2}
\end{figure}

\section{Related Work}
\subsection{Character-level Models}
Most of the past and current state-of-the-arts and impactful studies in the field of natural language processing \citep{devlin2018bert,radford2019language,lan2019albert,yang2019xlnet,brown2020language,clark2020electra,raffel2020exploring,reed2022generalist,taylor2022galactica} relies on subword-level tokenization \citep{sennrich2015neural,wu2016google,kudo2018sentencepiece}. These pervasive choices are due to the reasonable trade-off between robustness and efficiency of subword tokenization, but their limitations in several special environments wherein the data is domain-specific and/or its distribution shifts frequently are problems that have to be addressed at some point. While some subsequent studies have suggested improved techniques for subword-level tokenization \citep{provilkov2019bpe,he2020dynamic,hiraoka2021joint,wang2021multi}, many of them involve significant increases in computational costs and engineering efforts.

Character-level modeling has long been proposed as a promising alternative to the (sub)word-level language representations. Although the chronic long-range dependence issues of pure character-level models \citep{sutskever2011generating,graves2013generating,zhang2015character} solved by the adoption of the transformer architecture \citep{vaswani2017attention} as demonstrated in
\citep{belouadi2022bygpt5,xue2022byt5}, the quadratic time-complexity of self-attention became a new bottleneck for character-level language representation. Some recently proposed studies try to address this problem by downsampling the long character sequences to an acceptable length: \citet{tay2021charformer} and \citet{clark2022canine} utilize convolutional and non-parametric mean-pooling respectively, \citet{godey2022manta} leverages non-parametric max-pooling. All of them requires additional computations for explicit downsampling and character-level enrichment, thus causing corresponding overheads. The proposed elementwise embedding achieves the similar using a trivial tensor reshaping operation so does not degrades the inference speed of its backbone transformer architecture while processing much longer sequences.

\subsection{Efficient Transformers}
The major challenge of the transformer architecture \citet{vaswani2017attention} is to mitigate its quadratic self-attention complexity. \citet{liu2018generating} proposed to enhance this complexity by computing attentions within partitioned embedding matrices. \citet{child2019generating} estimated the full attention by mixing a sparse number of local attentions. \citet{beltagy2020longformer} extended this idea using technique called dilated sliding window to cover a wider range of attention. \citet{zaheer2020big} achieved the similar by attending three different types of attentions: global, windowed and random. \citet{kitaev2020reformer} alleviated the memory complexity by applying locality sensitive hashing and reversible residual layers \citet{gomez2017reversible}. \citet{wang2020linformer} improved overall attention complexity to be linear time by performing dimensionality reduction on the length axis, \citet{katharopoulos2020transformers} enhanced the computational complexity to linear time using kernel-based formulation and causal masking. 

Approaches through other structural modifications have also been proposed. \citet{lan2019albert} reduced the size of its BERT \citep{devlin2018bert} backbone extremely smaller by sharing weight parameters between every attention layer. The concept of knowledge distillation \citep{hinton2015distilling} has been demonstrated to be useful by \citep{sanh2019distilbert,tang2019distilling,jiao2019tinybert}.
Trials for improving the inference-time efficiency in which unimportant attention heads are pruned away \citep{voita2019analyzing,michel2019sixteen}; or "blocks" are pruned instead \citep{lagunas2021block}; were made. Some recent studies proposed to enhance computational complexity by downsampling input sequences to an acceptable length \citep{godey2022manta,tay2021charformer}. Elementwise embedding is quite similar to these approaches in terms of increasing the efficiency of transformer architecture, but is fundamentally different in that it does not require any downsampling and architectural modification of the transformer model to work with. What it does is that simply projects $uv$ bytes sequence into a $(uv, c)$ character embedding matrix, reshapes it into a $(u, w=vc)$ embedding matrix, and pass it to a transformer as input features; thus making it to align $uv$ sequences at the $O(u^{2})$ complexity. As elementwise embedding is designed as a pure embedding technique that does not modify any part of the transformer architecture, it can be potentially utilized with all the above methods in conjunction.

\subsection{Patent Classification}
\label{sec3.3}
Patent classification is an interesting subfield of text classification that targets to automize the categorization of patent documents.
Although this research field has not been actively studied compared to other compelling applied-ml areas like medicine and robotics, it deserves attention because it is essential for the data-driven patent analysis \citep{lee2009business,kim2017forecasting}. Patents filed during a specific period are often closely associated with technological trends at that time which implies a big flow of capital, and modern machine learning methods as large language models \citep{brown2020language,taylor2022galactica} and graph-powered neural networks \citep{sanchez2020learning,jumper2021highly,stokes2020deep} are powerful enough for extracting meaningful patterns from those textual/tabular data to lead high-impact decision makings.

In addition to the purpose of patent analysis, patents are also valuable for general-purpose machine learning research. Patent documents are large amounts of multi-modal data that consist of texts, graphs, images and their organized structure makes it easier to preprocess than dealing with raw data such as randomly crawled web corpora. Furthermore, patent data can be used to benchmark learning
algorithms because its extremely imbalanced distribution and a wide range of domain-specific lexicons hinder the models from convergence. This is the main reason that we utilize patent classification for
evaluating improved robustness of our models trained with the proposed elementwise embedding.

The technical requirement of patent classification is largely threefold. First, the classifier should be possible to capture the unique properties (i.e., the classification symbols) of each patent that is
computed by Precision and must be able to distinguish between different patent documents (i.e., the difference between classification symbols of two separate patent documents) that is calculated by Recall. Second, the classifier has to do well on both Precision and Recall, and every classification symbol should have equal importance (all technical categories are potentially important even if it is not currently popular). The former can be achieved by F\textsubscript{1} measure which is the harmonic mean of Precision and Recall, and the latter is satisfied by computing the micro-averaged scores for these three metrics. Third, the first-listed classification symbol must best indicate the invention of each patent and later ones offer additional information regardless of their relative positions (i.e., the meaning of a symbol differs by its order listed). To the best of our knowledge, however, this guide has not been reflected in known previous studies on multilabel patent classification \citep{lim2017ipc,li2018deeppatent,yadrintsev2018fast,lee2020patent,haghighian2022patentnet}. They considered two different patents having classification symbols [G06Q, G06Q, A01B], [A01B, G06Q, A01B] as identical, one-hot encoding their labels as [A01B: 1, G06Q: 1]. This skewed labeling no longer provide classifiers with the correct evaluation criteria. To fix this problem, we performed simple relabeling that conserves the order of class labels (see Section \ref{sec5.3}) and compared the experimental results with our own baselines. Section \ref{sec5.5} offers the equations of the above three metrics used in the following experiments.

\section{Methodology}
Before explaining the detailed implementation of elementwise embedding, we define mathematical notations to be used in this section. We denote the sequence of $u$ semantic units (i.e., the given text), as an embedding matrix $e_{u}\in\mathbb{R}^{u\times w}$ and its $i^{th}$ row (e.g., the $i^{th}$ word in a sentence) by $e^{(i)}$. The $j^{th}$ character in each $i^{th}$ semantic unit (e.g., the $j^{th}$ letter of $i^{th}$ word) is denoted by $e^{(i)[j]}\in\mathbb{R}^{1\times c}$, where $c=w/v$. Focus embeddings $f^{(i)}\in\mathbb{R}^{1\times w}$ (local) and $g^{(p)}\in\mathbb{R}^{1\times c}$ (global) are added to
$e^{(i)}$ and $e^{(i)[j]}$ respectively, by elementwise addition $\oplus$, where $p=(i\times j) + j$. Operator $\mathrm{Reshape}_{(a\times b)}$ reshapes the given tensor to be of the shape $(a\times b)$.

\subsection{Elementwise Embedding}
This section describes the detailed implementation of elementwise embedding, the technique for elementwise language representation. Consider a word with a missing spelling \emph{App\_le}. Missing spelling has low entropy since we can easily infer that it will be "\emph{l}". In the case of a sentence with a missing word "\emph{\_ brought a basket of apples to the front yard.}", the entropy of missing word becomes higher since its spellings vary depending on what word the subject becomes. For the same reason, the entropy of a missing sentence in a paragraph will skyrocket. Based on this intuition, we can assume that the entropy is lowest at the character-level and grows in higher semantic levels; the entropy of a semantic unit is proportional to the number of its spellings.
\begin{theorem*}
For a $v$ letters semantic unit $x$, $\mathrm{H(x)}\propto v$
\end{theorem*}
where $\mathrm{H}$ is the \emph{Shannon entropy} $\mathrm{H(x)}=-\mathbb{E}_{\mathrm{x}\sim p}[\mathrm{log}P(x)]$ and $v\in\mathbb{R}$. Because low entropy means there are fewer cases to encode, it is natural to represent a character as a much lower-dimensional latent embedding. Assuming a character as one of UTF-8 bytes
\footnote{While only ASCII characters can be expressed in 1 byte, we used this analogy for ease of explanation.} and each semantic unit (i.e., a token) consists of $v$ characters, a neural network with $w$-dimensional hidden layers will have 256 $c$-dimensional character embeddings as its embedding table when $c=w/v$. A semantic unit is abstracted into a horizontal concatenation of $v$ character embeddings. Larger meanings are just concatenations of smaller ones and this hierarchical expression allows neural networks to explicitly recognize characters while learning at any complex-level of semantics. We name this pair of $(256, c)$ embedding table and the following concatenation operation as elementwise embedding, referring 256 character embeddings to \emph{elements} and their concatenated meanings to \emph{materials}.

As entropy of any $v$ letters semantic unit increases proportionally with $v$, we need a way to reduce it again. One reasonable approach is to concentrate the probabilities of $v$ spellings to several important letters. Using self-attention \citep{vaswani2017attention} with $v$ attention heads, we can align $\{e^{(i)[n]}\}_{i=1}^{u}$, the sequence of the $n^{th}$ letters of $u$ semantic units, thereby assigning higher probabilities to more important characters. It is similar to that we often catch the meanings of words by focusing only some morphologically noticeable spellings, so we call this operation \emph{focus}. By setting the number of attention heads to $h=v$, $n^{th}$ attention focuses on to $n^{th}$ important letters when ${n\in[1, v]}$. For example, when the input sequence is encoded to a sequence [\emph{Focus}, \emph{on}, \emph{the}, \emph{elements}], $1^{st}$ attention head attend to $1^{st}$ spellings [\emph{F}, \emph{o}, \emph{t}, \emph{e}] focusing on the most important letters e.g., [\emph{F}, \emph{e}], $2^{nd}$ head attends to [\emph{o}, \emph{n}, \emph{h}, \emph{l}], and so on (see Fig \ref{fig5} in Appendix A).
\begin{prop}
The entropy of an important semantic unit $e^{(i)}$ in the given text $e_{u}$ can be minimized using $v$-headed self-attention, where $e^{(i)}$ consists of horizontally concatenated $v$ $c$-dimensional character embeddings.
\end{prop}
\begin{proof}
In forward propagation, each $n^{th}$ letter in $i^{th}$ semantic unit $e^{(i)[n]}$ is assigned a probability by softmax function in $n^{th}$ attention head, then concatenated back ($n\in[1, v], i\in[1, u]$). All $uv$ characters are jointly projected by a position-wise feed-forward layer. This allows neural networks to jointly attend once to spellings of each $e^{i}$ and once to the entire $uv$ characters in $e_{u}$. Probability of each $e^{(i)}$ is determined by the alignment of its spellings $\{e^{(i)[n]}\}_{n=1}^{v}$. Loss is computed by an arbitrary objective function and errors are backpropagated to each $e^{(i)[n]}$. Networks and elementwise embedding are updated to assign higher probability to more crucial $e^{(i)}$, based on its letters $\{e^{(i)[n]}\}_{n=1}^{v}$, so that the given objective is minimized.
\end{proof}

Note that by matching the number of attention heads with $v$, we can restrict the subspace of each attention head to the latent space of $c$-dimensional character embeddings (i.e., elements), when attention layers expect $vc=w$-dimensional embeddings (i.e., materials) as input features. It can be interpreted as approximating the $w$-dimensional latent space with closed set of $w/v=c$-dimensional vectors, and also as separating the roles of embeddings and hidden layers: embeddings to encode character-level semantics and hidden layers to encode more complex-levels of semantics. This helps neural networks do not waste their limited expressiveness on encoding implicit information. Because the transformer architecture \citep{vaswani2017attention} is itself multiple layers of attention, we do not have to implement $focus$ operation by hand, so what we need for elementwise language representation is only a single-time tensor reshaping operation from $(uv, c)$ to $(u, w)$; which is the same as concatenating $uv$ c-dimensional elements to be $u$ $w$-dimensional materials. Following this framework, the standard transformer architecture can align $uv$ characters at the $O(u^{2})$ computational complexity fully ignoring the value of $v$. The acceptable length of the input sequence scales with the size of the hidden layers so larger models can process longer sequences than the smaller ones at the same attention complexity; this is more natural than the current transformers that no matter how large the model is, the length of the sequence it can process in a reasonable amount of time does not change at all.

\subsection{Implementation Details}
\label{sec4.2}
In practical implementation, focus for aligning important elements is performed by the mutli-head attention of parent architecture (i.e., a transformer model to work with) having $h=v$ attention heads, so
elementwise embedding is just a set of lookup table containing 256 $c$-dimensional element embeddings and the following tensor reshaping operation. In other words, applying elementwise embedding is
simply to replace the existing embedding table of parent model with elementwise embedding. While token embeddings are usually trained with neural networks, they can be always detached fully
independently of the network architecture so that we do not regard replacing embedding table as a structural modification. This feature is clearly different from the previous studies on character-aware
language representation that additional computational components (e.g., non-parametric mean/min/max pooling, shallow convolutional/recurrent/transformer layers) are required for enriching character-level information or up/downsampling input sequences as leveraged in \citep{ma2020charbert,tay2021charformer,clark2022canine,godey2022manta}.

Technically, elementwise embedding is implemented by a single-time reshaping operation as,
\[e^{(i)}\leftarrow\ \mathrm{Reshape}_{(1,w)}[\{e^{(i)[j]}\oplus g^{(p)}\}_{j=1}^{v}]\oplus f^{(i)}\]
We add two kinds of position embeddings $g^{(p)}$ and $f^{(i)}$ called "focus embeddings" to $e^{(i)[j]}$ and $e^{(i)}$ respectively, to manually encode the focusable positions: the former describes global positions (e.g., position of a character $e^{(i)[j]}$ in the entire sentence $e_{u}$) and the latter directs the local positions (e.g., position of a spelling $e^{(i)[j]}$ in the $i^{th}$ word $e^{(i)}$). Though elementwise embedding works well without focus embeddings, we found that they help models trained with elementwise embedding a lot to converge more stable and to perform better as explained in Section \ref{sec6.1}. Before being passed to as input to the parent model, dropout \citep{srivastava2014dropout} and normalization \citep{ba2016layer} can be applied to embeddings for better generalization performance.

Elementwise embedding applies with any type of tokenization by following the framework:
\begin{enumerate}[topsep=0pt,noitemsep]
\item Divide given text into $u$ tokens
\item Encode each token into $v$ integers (UTF-8 bytes)
\item Project integers into a $(uv, c)$ element embedding matrix
\item Reshape element embedding matrix into a $(u, w)$ material embedding matrix
\end{enumerate}
where $w=vc$ (see Fig \ref{fig2}).

\section{Experimental Setup}
\subsection{Dataset}
All models used in the following experiments were trained on patent documents published by the USPTO (United States Patent and Trademark Office) from 2006 to 2014 and then tested on the
two test set splits 2015\textsubscript{A} and 2015\textsubscript{B} that consist of patents in the first- and second-half of 2015, respectively. We leveraged each patent document as a single training example:
its claim texts as input features and the classification symbols (i.e, CPC codes) assigned to it as labels. Among the five hierarchical levels of Section, Class, Subclass, Main Group, Subgroup, we used the slices from Section to Subclass (i.e., subclass-level CPC codes) as labels: label of the CPC code A01N 53/12 [Section: A, Class: 01, Subclass: N, Main Group: 53/00, Subgroup: 53/12] is A01N which is the 
concatenation of slice [Section: A, Class: 01, Subclass: N]. We concatenated the first 20 claims into a single input text, instead of using the first claim only as in \citet{lee2020patent} to prevent the classifier from overfitting on meaningless training examples that are too short to describe each patent's own invention. The entire number of class labels are doubled from 664 from 1,328 by the two position attributes
\emph{First} and \emph{Later} after relabeling, decreasing the standard deviation of every dataset split (see Section \ref{sec5.3}) which means that the imbalances between class labels are quite mitigated. Because we used patent data to benchmark the robustness of classification models to domain-specificity and long-tailed distribution, we did not adjust the dataset imbalance further.  We collected utility type patents only
filed in the United States as training data from the BigQuery table called Google Patents Public Data \footnote{\url{https://github.com/google/patents-public-data}}. Table \ref{table2} provides a statistical summary of the dataset. 

\begin{table}[h]
\caption{Configurations of all six models used in the following experiments. Hyperparameters $u$ and $v$ denote the number of materials and of elements per a material. $w$ and $c$ refer to the size of materials and of elements, respectively. $w$ is always a multiple of $v$ and $c$. Hyperparameter $v$ is also used for ORIG models for intuitive comparison; ORIG models represent a material using one element. $h$ denotes the number of attention heads.}
\label{table1}
\begin{center}
\begin{tabular}{llllllll}
\hline 
\multirow{2}{*}{\bf Model}          & \bf Parameters       &                  & \multirow{2}{*}{\emph{\textbf u}}   & \multirow{2}{*}{$\emph{\textbf v}$}        & \bf\multirow{2}{*}{\emph{\textbf w}}      & \bf\multirow{2}{*}{\emph{\textbf c}}     & \bf\multirow{2}{*}{\emph{\textbf h}}   \\
                                                  & \bf Total            & \bf Embedding       \\                                          
\hline 
BERT\textsubscript{EWE}        & 87M (0.8x)         & 12k (0.005x)   & 128                                 & $16 (=h)$                        & 768                            & $48 (=w/v)$                & $16 (=v)$ \\
BERT\textsubscript{ORIG}       & 110M (1x)          & 23M (1x)         & 128                                 & 1                                      & $768 (=c)$                & 768                             & 12  \\
\hline
ALBERT\textsubscript{EWE}    & 9M (0.7x)           & 2k (0.005x)     & 128                                 & $16 (=h)$                         & 128                           & $8 (=w/v)$                  & $16 (=v)$ \\
ALBERT\textsubscript{ORIG}   & 12M (1x)            & 4M (1x)           & 128                                 & 1                                      & $128 (=c)$                & 128                             & 12  \\
\hline
CANINE\textsubscript{EWE}    & 110M (0.8x)       & 1M (0.05x)      & 128                                 & $16 (=h)$                         & 768                           & $48 (=w/v)$                & $16 (=v)$ \\
CANINE\textsubscript{ORIG}   & 130M (1x)          & 25M (1x)         & 128                                 & 1                                       & $768 (=c)$                & 768                            &  12 \\
\hline 
\end{tabular}
\end{center}
\end{table}

\begin{figure}[t!]
\begin{center}
\includegraphics[width=1.0\linewidth]{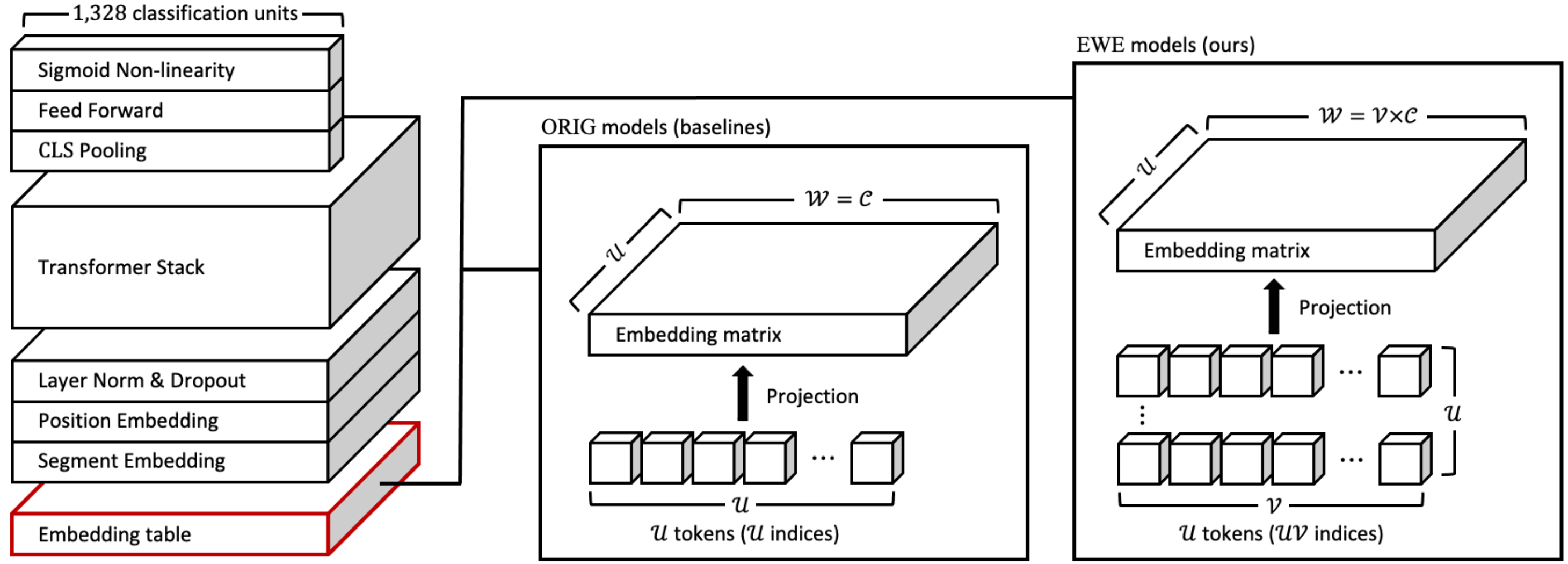}
\end{center}
\caption{Visualization of the model architectures used in the experiments. As shown in the above graphic, applying elementwise embedding is simply to replace the embedding table of any transformer-based model. Unique structures of CANINE and ALBERT and specific implementation of the standard transformer stack were omitted for visual brevity.
}
\label{fig4}
\end{figure}

\subsection{Model Architectures}
As mentioned repeatedly in previous sections, applying elementwise embedding is to replace the embedding table of the parent transformer architecture. We trained three different transformers 
BERT \citep{devlin2018bert}, ALBERT \citep{lan2019albert}, CANINE \citep{clark2022canine} with elementwise embedding for multilabel patent classification and compared with their original implementations to demonstrate the idea of elementwise language representation. We denote three transformers with elementwise embedding by EWE models such as BERT\textsubscript{EWE},
ALBERT\textsubscript{EWE}, CANINE\textsubscript{EWE} and their original versions, the baselines, as ORIG models like BERT\textsubscript{ORIG}, ALBERT\textsubscript{ORIG},
CANINE\textsubscript{ORIG}. In both theoretical and practical implementations, elementwise embedding does not modify any part of its parent model so that EWE models always have exactly the same 
architectures as their ORIG equivalences. Every model used in our experiments follow the configuration of BERT\textsubscript{BASE} \citep{devlin2018bert} wherein a transformer encoder has 12 768-dimensional attention layers with 12 heads \footnote{While EWE models use 16 attention heads, this is only to meet the theory of elementwise embedding so that is independent of their superior performances to the ORIG models as shown in the ablation study in Section \ref{sec6.2}} each followed by a 3072-dimensional linear projection; every model processes $u=128$ tokens at once but EWE models align $v$ times longer sequence than ORIG models since each token is represented by $v$ embeddings (see Fig \ref{fig4}). Table \ref{table1} shows the differences in between configurations of EWE and ORIG models.

\subsection{Patent Relabeling}
\label{sec5.3}
For the purpose of patent classification, the label of each patent document becomes a list of textual symbols classifying it to 
the corresponding technical categories. Among several classification schemes, the two most frequently utilized from previous studies are
IPC (International Patent Classification) and CPC (Cooperative Patent Classification). One interesting fact is that the assigned symbols have 
different meanings depending on the order in which they are listed: the first-listed symbol describe the invention of each patent document
and the later symbols represent additional information, so that patent documents with classifications [G06Q, G06Q, A01B], [A01B, G06Q, A01B], [G06Q, A01B]
should have three different labels. This is in accordance of  their documentations \citep{uspto,wipo}, but to the best of our knowledge, previous studies on patent classification
\citep{lim2017ipc,li2018deeppatent,yadrintsev2018fast,lee2020patent,haghighian2022patentnet} did not reflect this guide, hence one-hot encoding the above three patents to 
an identical label [A01B: 1, G06Q: 1] (see the left of Fig \ref{fig3}). This skewed labeling is quite undesirable for both practical patent classification and algorithm benchmarking since the metrics for
evaluating the classifier will be messed up by the distorted one-hot encodings.

To address this problem, we simply relabeled our patent documents by attaching position attributes \emph{First} and \emph{Later} as prefixes to their classification symbols.
We relabeled the above classifications [G06Q, G06Q, A01B], [A01B, G06Q, A01B], [G06Q, A01B] as [First-G06Q, Later-G06Q, Later-A01B], [First-A01B, Later-G06Q, Later-A01B],
[First-G06Q, Later-A01B] so that to be one-hot encoded as [0, 1, 1, 1], [1, 0, 1, 1], [0, 1, 1, 0]; when the placeholder for one-hot encoding is [First-A01B, First-G06Q, Later-A01B, Later-G06Q] (see the right of Fig \ref{fig3}). Relabeled symbols fully satisfy the technical requirements for patent classification described in Section \ref{sec3.3} and furthermore, the imbalanced distribution between class labels were
significantly alleviated since the count of each classification symbols (CPC codes here) was halved by position attributes \emph{First} and \emph{Later} (See Table \ref{table2}). As a result, all our models trained on relabeled data do well both on Precision and Recall (see Table \ref{table3}), compared to the case of aforementioned studies that are biased against either score. 

\begin{table}[t!]
\caption{A statistical summary of dataset used in the experiments. \emph{Total} is the standard deviation between all CPC codes (class labels), \emph{Majors} denotes the standard deviation between CPC codes that dominate over 90\% of the entire labels in training examples. Standard deviations in all three dataset splits were significantly decreased after relabeling (see the second row of the table) which means that the imbalances between class labels were alleviated.}
\label{table2}
\begin{center}
\begin{tabular}{llllllllll}
\hline 
\multirow{2}{*}{\bf Labeling}  & \bf Train        &                                             &                       & \bf 2015\textsubscript{A} &               &                 & \bf 2015\textsubscript{B} &               &              \\
                                              & \bf Samples  & \bf Total                                & \bf Majors     & \bf Samples                     & \bf Total & \bf Majors & \bf Total                           & \bf Total & \bf Majors \\                                          
\hline
Raw                         & \multirow{2}{*}{1.9M} & 21k                                      & 35.4k   &\multirow{2}{*}{14.5k}  & 1.8k                             & 3k          & \multirow{2}{*}{15.4k}  & 1.9k              & 3.2k \\
Relabeled                &                                   &\underline{12.6k}                  & \underline{21k}  &                      & \underline{1.1k}          & \underline{1.9k} &                         & \underline{1.2k} & \underline{2k} \\ 
\hline

\end{tabular}
\end{center}
\end{table}

\begin{figure}[t!]
\begin{center}
\includegraphics[width=1.0\linewidth]{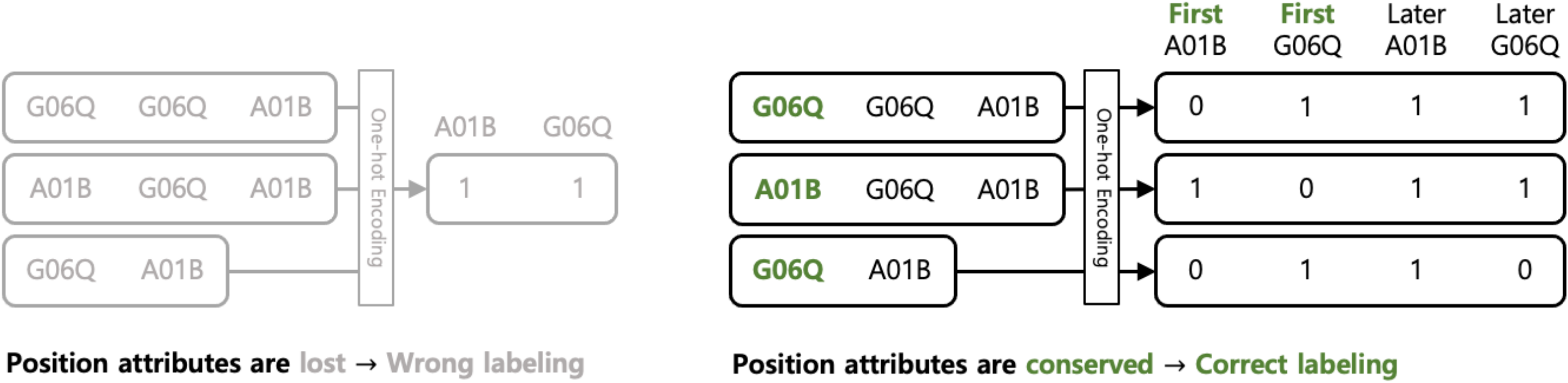}
\end{center}
\caption{Visualization of correct labeling which separates the first-listed and later classification symbols. Our proposed relabeling strategy (right) conserves the order between symbols (First- and Later-CPC codes) even after being one-hot encoded that are ignored in previous literature (left).
}
\label{fig3}
\end{figure}

\subsection{Training Details}
This section provides specific configurations to train our models. Every model was trained during 10 epochs using binary cross-entropy loss with sigmoid activation ($threshold=0.3$) for multilabel
patent classification from scratch without pretraining. We used AdamW \citep{loshchilov2017decoupled} as optimizer ($\beta_1=0.9$ and $\beta_2=0.999$, $eps=1e-8$) with $L_2$ regularization
($\lambda=0.01$). Initial learning rate decays linearly from $2e-5$ without a warmup period and batch size was set to 32; we selected these small values to prevent overfitting due to imbalanced 
training examples. Larger memory-safe batch sizes did not show meaningful differences, only slowing convergence. All models were trained and tested using a single 24GB VRAM GPU (NVIDIA RTX TITAN) with FP16 mixed-precision.

\subsection{Performance Measures}
\label{sec5.5}
For evaluating models for multilabel patent classification, we utilize three metrics Precision, Recall and their harmonic mean F\textsubscript{1} measure. Because all CPC codes (technical categories of
patents; class labels) are equally important (current technical categories draw long tail when counted by the number of the patent documents that they are classifying; only 30\% of the categories contain over
90\% of the entire patents, mainly due to the popularity of each field of technology; but all of them are potentially valuable) we compute micro-averaged scores (TP, TN, FP, FN) for all metrics by
$\mathrm{Score}=\sum\nolimits_{l} \mathrm{Score}_{l}$.
Since we consider a total of 664 CPC symbols as class labels with two position attributes \emph{First} and \emph{Later}, $l\in[1, 1328]$ is established. As mentioned in Section \ref{sec3.3},
$\mathrm{Precision=TP\cdot(TP+FP)^{-1}}$ captures how well a model identifies the unique classifications of each patent document, $\mathrm{Recall=TP\cdot(TP+FN)^{-1}}$ implies how well a model distinguishes between different patent documents which is critical for patent search engine optimization. As Precision and Recall are equally important for  patent classification, we utilize their $\mathrm{F_{1}=2\cdot(Precision^{-1}+Recall^{-1})^{-1}}$ as the main metric.

\begin{table}[t]
\caption{Superiority of transformers trained with elementwise embedding in multilabel patent classification.}
\label{table3}
\begin{center}
\begin{tabular}{llllllll}
\hline 
\multirow{2}{*}{\bf Model}          & \bf 2015A                   &                                          &                                          & \bf 2015B                   &                                     & \\
                                                  & \bf F1                         & \bf Precision                     & \bf Recall                          & \bf F1                         & \bf Precision                & \bf Recall \\
\hline 
BERT\textsubscript{EWE}        &\underline{64.30}         &\underline{66.02}               & \underline{62.66}             & \underline{63.94}       & 66.55                           & \underline{61.53} \\
BERT\textsubscript{ORIG}       &63.68                           &65.59                                  & 61.82                               & 63.35                         & \underline{67.16}         & 59.95 \\
\hline
CANINE\textsubscript{EWE}    &\underline{64.30}         &\underline{65.86}               & \underline{62.82}             & \underline{63.95}       & \underline{66.43}          & \underline{61.64} \\
CANINE\textsubscript{ORIG}   &60.40                           &64.08                                  & 57.12                               & 59.97                         & 64.52                            & 56.01 \\
\hline
ALBERT\textsubscript{EWE}    &\underline{63.18}         &\underline{65.84}               & \underline{60.73}             & \underline{62.91}       & \underline{66.47}          & \underline{59.71} \\
ALBERT\textsubscript{ORIG}   &63.15                           &65.82                                  & 60.70                               & 62.79                         & 66.36                            & 59.59 \\
\hline 
\end{tabular}
\end{center}
\end{table}

\section{Results}
This section presents the experimental results demonstrating the validity of the proposed method explained so far. Because there are no comparable state-of-the-arts in multilabel patent classification that comply the guide to the position attributes (see Section \ref{sec3.3}), we compare the results with our own baselines. Every EWE model was trained using whitespace tokenizer for a fair comparison with subword-level models (BERT\textsubscript{ORIG} and ALBERT\textsubscript{ORIG}); we analyze tokenization-free EWE models in ablation study in Section \ref{sec6.2}. As shown in Table \ref{table3}, all EWE models surpass their corresponding baselines in multilabel patent classification on all test-set splits 2015\textsubscript{A} and 2015\textsubscript{B}. Used patent dataset is highly imbalanced and contain massive amounts of unusual technical lexicons, so that the superior classification performances of EWE models on it show clear robustness to domain-specificity and long-tailed distributions improved by elementwise language representation. Note that all EWE models have much less embedding parameters than their original implementations.

ALBERT\textsubscript{EWE} and CANINE\textsubscript{EWE} are both improved by elementwise embedding while maintaining their own design choices: the shallow transformer layer for character-level encoding and 1D strided convolutional layers for sequence downsampling of CANINE \citep{clark2022canine}, and the factorized embedding parameterization and parameter sharing of ALBERT \citep{lan2019albert}. These empirically demonstrate the generalizability of elementwise language representation. Performance enhancement in ALBERT\textsubscript{EWE} is relatively smaller than in other EWE models that is presumably because it uses much lower-dimensional embeddings as elements $(c=8)$ than others (with $c=48$). Based on this observation, using $c$ greater or equal than 8 is recommended. CANINE\textsubscript{EWE} shows the largest improvement among all EWE models, however, its classification performance is at the same level as BERT\textsubscript{EWE} that has 40\% fewer parameters, so it is unclear whether additional sequence downsampling gives a meaningful benefit when elementwise embedding is already applied. All EWE models process $v=16$ times longer sequences than their ORIG counterparts at the same $O(N^{2})$ computational complexity. The overhead of the one-time tensor reshaping operation for elementwise language representation is negligible and it can even be removed by technical optimizations such as JIT (Just In Time) compilation, so there is no meaningful difference in inference speeds between EWE and ORIG models\footnote{Although reshaping tensors in GPU entails physical arrangements so is relatively more expensive than in CPU, the cost of reshaping from $(uv, c)$ to $(u, w=vc)$ is trivial and can be optimized well by technical tricks like asynchronous dispatch and JIT compilation of XLA; so elementwise embedding does not slower its parent transformer.}. 

\begin{table}[h]
\caption{Effect of focus embeddings on the convergence of transformers trained with elementwise embedding.}
\label{table4}
\begin{center}
\begin{tabular}{llllllll}
\hline 
\multirow{2}{*}{\bf Model}          & \multirow{2}{*}{\bf Ablation}       & \bf 2015A                   &                                        &                                          & \bf 2015B                   &                                & \\
                                                  &                                                    & \bf F1                         & \bf Precision                   & \bf Recall                          & \bf F1                         & \bf Precision           & \bf Recall \\
\hline 
\multirow{2}{*}{BERT\textsubscript{EWE}}       & None                                         &\underline{64.30}         & 66.02       & \underline{62.66}              & \underline{63.94}       & 66.55         & \underline{61.53} \\
       & Focus embeddings                                                                         & 63.22                          & \underline{66.14}      & 60.55                                & 62.91                         & \underline{66.69}         & 59.54 \\
\hline
\end{tabular}
\end{center}
\end{table}

\subsection{Effect of Focus Embeddings}
\label{sec6.1}
Originally, the main design goal of focus embedding was to stabilze the convergence of models trained with elementwise embedding. While the idea of elementwise embedding works well without focus embeddings (see Table \ref{table4}), training parent models with it was somewhat tricky; the starting point of meaningful convergence differed randomly at each training trial, and therefore the reproducibility was not ensured. Focus embeddings guarantee the stable convergence of elementwise models by explicitly encoding the global and local positions of focusable characters. We found that focus embeddings also improve classification performance of parent transformer, so set them as the default component of elementwise embedding. In the case where the positional information is supplemented 
in some other way e.g, $n$-gram as in tokenization-free BERT\textsubscript{EWE} (see Section \ref{sec6.2}) focus embedding did not provide a meaningful enhancement in both stability and performance, however, more research is needed on how their absence will affect other applications as sequence-to-sequence modeling. Table \ref{table4} shows the ablation study on focus embeddings BERT\textsubscript{EWE} classifier.

\begin{table}[h]
\caption{Effect of other tokenization on the convergence of transformers trained with elementwise embedding.}
\label{table5}
\begin{center}
\begin{tabular}{llllllll}
\hline 
\multirow{2}{*}{\bf Model}          & \multirow{2}{*}{\bf Tokenization}     & \bf 2015A                   &                                        &                                          & \bf 2015B                   &                                & \\
                                                  &                                                        & \bf F1                         & \bf Precision                   & \bf Recall                          & \bf F1                         & \bf Precision           & \bf Recall \\
\hline 
\multirow{3}{*}{BERT\textsubscript{EWE}}        & None                          &60.01         &63.77            & 56.67              & 59.80       & 64.49                    & 55.75 \\
        & Gradient                                                                                       &64.14         &65.91            & 62.45              & 63.75       & {66.41}                  & 61.29 \\
        & Whitespace                                                                &\underline{64.30}        &\underline{66.02}               & \underline{62.66}              & \underline{63.94}       & \underline{66.55}                    & \underline{61.53} \\
\hline
\end{tabular}
\end{center}
\end{table}

\subsection{Effect of Tokenization Strategies}
In this section, we generalize the idea of elementwise embedding to tokenization-free language representation. We implement two versions of tokenization-free BERT\textsubscript{EWE}: one using pure UTF-8 byte-level tokenizer\footnote{UTF-8 bytes encoding can be achieved by \texttt{list(bytes("text/to/encode", "utf-8"))} in Python3.} and one using gradient-based tokenizer. Tokenization-free elementwise language representation follows the same framework described in Fig \ref{fig2} and Section \ref{sec4.2} \footnote{For tokenization-free elementwise language representation, text is encoded to a sequence of $uv$ UTF-8 bytes, projected into a $(uv, c)$ element embedding matrix, then reshaped into a $(u, w)$ material embedding matrix directly without any tokenization.}. First, BERT\textsubscript{EWE} trained using raw UTF-8 bytes encoding (see the first row of Table \ref{table5}) shows significantly poor performance compared to one trained using whitespace tokenizer (see the third row Table \ref{table5}), but still equivalent to CANINE\textsubscript{ORIG} which is 1.3x larger. This version of BERT\textsubscript{EWE} shares the same hyperparameters with whitespace version which set $n=v=16$ and  $u=128$. 

Notably, tokenization-free BERT\textsubscript{EWE} recovers the same level of performance as one with whitespace tokenizer when trained using gradient-based tokenizer (see the second row of Table \ref{table5}). In this implementation, each element is replaced with softmax-weighted sum-pooled $v$-gram\footnote{This can be understood as the simplified version of \emph{Block scoring network} proposed in \citet{tay2021charformer}.}:
\[e^{(i)[j]}\leftarrow\sum_{j=1}^{v}\alpha_{j}e^{(i)[j]}\]  where $\alpha_{j}=\mathrm{softmax}_{j}(e^{(i)[j]}s)$ and $s\in\mathbb{R}^{c}$ is the weight vector for linear projection
$\mathbb{R}^{c} \mapsto \mathbb{R}$.
Since aggregated $v$-grams implicitly encode the positions of focusable elements, focus embeddings are not used for this model. $n=v=8$ attention heads are used. Operations for $v$-gram pooling cause overhead but is trivial, so that the resulting decrease in inference speed is negligible. This type of elementwise representation is expected to be useful for the tasks where consistent tokenization is difficult (e.g., multilingual applications which deal with heterogeneous language systems simultaneously). For large language models such as GPT-3 (Brown et al., 2020) with much wider hidden layers (so each semantic unit can have much greater numbers of characters), however, whitespace tokenization will suffice for almost all kinds of language representations.
\label{sec6.2}

\section{Discussion}
So far, we explored the new elementwise language representation from both theoretical and practical aspects.

This framework gives advantages of:
\begin{itemize}[topsep=0pt,noitemsep]
\item being generalized to every level of language representation
\item being able to process longer sequences at the same complexity as model scales
\item being able to reuse existing transformer architectures for all levels of language representation
\end{itemize}
Neither architectural modification nor additional computational overheads occured.

This framework remains several challenges that:
\begin{itemize}[topsep=0pt,noitemsep]
\item does not reflect various other linguistic components than semantics
\item has not yet proposed an optimal strategy for unsupervised pretraining
\item still requires separate embedding parameters for language representation 
\end{itemize}

This framework suggests new research directions which can be extended as follows:
\begin{enumerate}[topsep=0pt,noitemsep]
\item pretraining a single language model understanding all kinds of tokenization
\item pretraining a language model with multiple levels of semantics at the same time
\item integrating representations of all types of languages (text, image, graph, etc.) into bytes
\end{enumerate}

Recent groundbreaking studies \citep{brown2020language,raffel2020exploring,jaegle2021perceiver,reed2022generalist} have successfully demonstrated that it is possible to 
pretraining multimodal, multitasking neural networks. By expanding their contributions with these research directions, computers will finally be able to understand the world 
solely based on their native language, bytes.

\section{Conclusion}
In this paper, we proposed elementwise embedding that is a technique for generalized language representation.
To the best of our knowledge, this is the first case of computational language representation:
\begin{itemize}[topsep=0pt,noitemsep]
\item applying with all levels of tokenization strategies
\item aligning longer sequences proportional to the model size
\item reusing existing transformers for all levels of language representations
\end{itemize}
without either any architectural modification or degradation in performance and inference speed.

We expand these contributions in our follow-up studies discussing:
\begin{itemize}[topsep=0pt,noitemsep]
\item elementwise representation for other types of data as images and graphs 
\item unsupervised pretraining approach for elementwise language representation
\item language modeling method leveraging multiple levels of semantics simultaneously 
\end{itemize}

\bibliography{main}
\bibliographystyle{tmlr}

\appendix
\section{Appendix}
\begin{figure}[h]
\begin{center}
\includegraphics[width=1.0\linewidth]{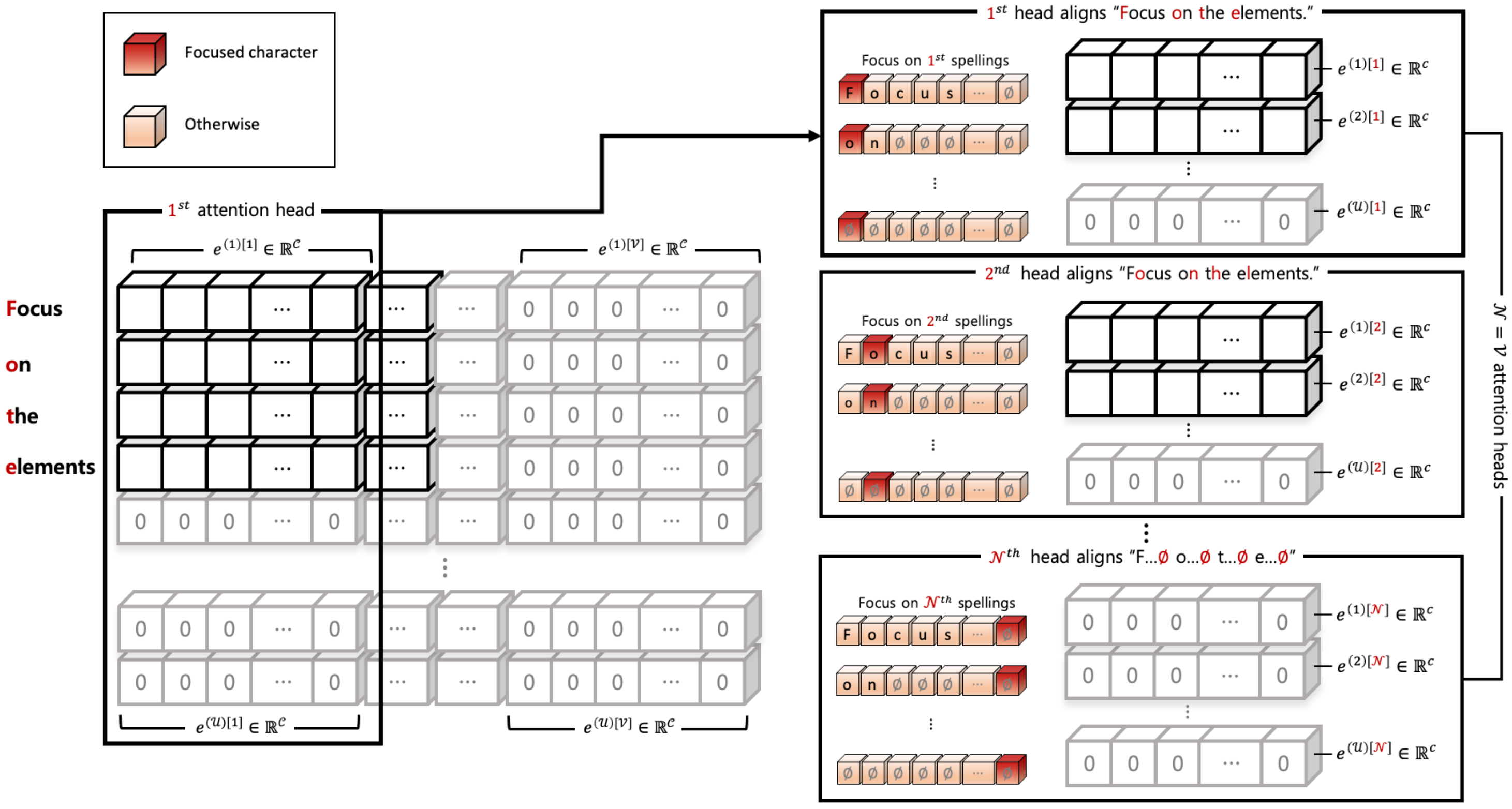}
\end{center}
\caption{Visualization of the focus operation in multihead attention. $n^{th}$ attention head aligns $n^{th}$ characters (i.e., elements) hence focusing on the most important ones, when the number of attention heads is set to $v$.
}
\label{fig5}
\end{figure}

\end{document}